\documentclass{article}
\usepackage[letterpaper,top=1in,bottom=1in,left=1in,right=1in]{geometry}
\usepackage{amsfonts}       

\usepackage{amsthm}
\usepackage{mathtools}      
\usepackage{amssymb}        
\usepackage{filecontents}
\usepackage{graphicx}       
\usepackage{marginnote}     
\usepackage{marvosym}       
\usepackage{overpic}        
\usepackage{tabularx}
\usepackage{cite}
\usepackage{color}
\usepackage{epstopdf}
\usepackage[linesnumbered,ruled,vlined]{algorithm2e}
\usepackage[normalem]{ulem}
\usepackage{epstopdf}
\usepackage{enumitem}
\usepackage[normalem]{ulem}
\usepackage{url}

\usepackage{diagbox}
\usepackage{multirow}

\newif\ifdraft
\draftfalse

\ifdraft
\usepackage[paperheight=11in,paperwidth=9.5in,
			left=1.25in,right=1.25in,
			top=0.75in,bottom=0.75in,
			heightrounded,marginparwidth=1.2in,
			marginparsep=0.05in]{geometry}
\usepackage{xcolor}
\usepackage{xargs} 
\usepackage[textsize=footnotesize]{todonotes}
\newcommandx{\nt}[2][1=]{\todo[linecolor=red,
			backgroundcolor=red!10,bordercolor=red,#1]{#2}}
\newcommandx{\sh}[2][1=]{\todo[linecolor=blue,
			backgroundcolor=blue!10,bordercolor=blue,#1]{SH:#2}}
\newcommandx{\sw}[2][1=]{\todo[linecolor=orange,
			backgroundcolor=orange!10,bordercolor=orange,#1]{SF:#2}}
\newcommandx{\jy}[2][1=]{\todo[linecolor=green,
			backgroundcolor=green!10,bordercolor=green,#1]{JY:#2}}
\else
\newcommand{\nt}[1]{{}}
\newcommand{\sh}[1]{{}}
\newcommand{\sw}[1]{{}}
\newcommand{\jy}[1]{{}}
\fi

\newif\iftwocolumn
\twocolumnfalse

\newtheorem{proposition}{Proposition}[section]

\theoremstyle{definition}

\theoremstyle{remark}
\newtheorem*{remark}{Remark}

\SetKwProg{Fn}{Function}{}{}
\SetKwComment{Comment}{$\triangleright$\ }{}

\makeatletter
\def\subsubsection{\@startsection{subsubsection}
                                 {3}
                                 {\z@ \hspace*{1mm}}
                                 {0ex plus 0.1ex minus 0.1ex}
                                 {0ex}
                                 {\normalfont\normalsize\itshape}}
\makeatother

\newcommand{\pnp}{\textsc{PnP}\xspace}

\newcommand{\W}{\mathcal{W}}

\newcommand{\R}{\mathbb{R}}

\newcommand{\getpnp}{\textsc{GetPnPTime}\xspace}
\newcommand{\algperm}{\textsc{OptSeq}\xspace}
\newcommand{\algdp}{\textsc{OptSeqDP}\xspace}
\newcommand{\alglocal}{\textsc{SubOptDP}\xspace}
\newcommand{\algspt}{\textsc{SPT}\xspace}
\newcommand{\alge}{\textsc{Euclidean}\xspace}

\newcommand{\algfifo}{\textsc{FIFO}\xspace}

\title{
Toward Fast and Optimal Robotic Pick-and-Place on a Moving Conveyor
}
\author{
Shuai D. Han \quad Si Wei Feng \quad Jingjin Yu
\thanks{
S. D. Han, S. W. Feng, and J. Yu are with the Department of Computer 
Science, Rutgers, the State University of New Jersey, Piscataway, NJ, 
USA. E-Mails: \{{\tt shuai.han, siwei.feng, jingjin.yu}\}\hspace*{.25em}
\MVAt \hspace*{.25em}{\tt rutgers.edu}. 
}
}

\begin{document}
\maketitle
\thispagestyle{empty}
\pagestyle{empty}

\vspace*{-5mm}
\begin{abstract}
Robotic pick-and-place (\pnp) operations on moving conveyors find a 
wide range of industrial applications. In practice, simple greedy
heuristics (e.g., prioritization based on the time to process 
a single object) are applied that achieve reasonable efficiency. We 
show analytically that, under a simplified telescoping robot model, 
these greedy approaches do not ensure time optimality of \pnp 
operations. To address the shortcomings of classical solutions, 
we develop algorithms that compute optimal object picking sequences 
for a predetermined finite horizon. Employing dynamic programming 
techniques and additional heuristics, our methods scale to up to 
tens to hundreds of objects. 
In particular, the fast algorithms we develop come with running time 
guarantees, making them suitable 
for real-time \pnp applications 
demanding high throughput. 
Extensive evaluation of our algorithmic solution over dominant 
industrial \pnp robots used in real-world applications, i.e., Delta 
robots and Selective Compliance Assembly Robot Arm (SCARA) robots, 
shows that a typical efficiency gain of around $10$-$40\%$ over 
greedy approaches can be realized. 
\end{abstract}

\section{Introduction}\label{sec:intro}
We present a study aiming at developing fast algorithms
for optimal robotic pick-and-place (\pnp) on a moving conveyor. Modeling 
typical industrial robotic \pnp scenarios, we examine the setting where 
a robotic arm is tasked to continuously pick up objects, one at a time, 
from a moving conveyor and drop them off at a fixed location. 
Based on our investigation, it would appear that greedy approaches had 
been used in practice because of the fast online nature of the task, which is explained in two aspects. 
First, estimating the poses of multiple objects requires advanced sensing techniques, whereas it is much 
easier to detect an object as it enters the scene (e.g., by using a laser scanner). Nowadays, however, computer vision 
algorithms are fast enough to accurately report the poses of 
many objects. Second, on a fast-moving conveyor, very limited 
computation can be done before an object becomes inaccessible. 

In this paper, we first work with a simplified robot model to show 
analytically that commonly used greedy approaches do not produce time-optimal 
solutions in general. Then, we develop 
dynamic programming based algorithms capable of computing (near-)optimal
solutions for tens to hundreds of objects in under a second. 
Because the running time can be accurately bounded for a given number of 
objects, our algorithmic solution can be customized for real-time \pnp 
operations. Extensive simulation studies on both simplified and practical 
robot models including Delta and SCARA (Selective Compliance Assembly 
Robot Arm) robots show that our proposed methods consistently yield about 
$10$-$40\%$ efficiency gain with respect to the number of objects that 
can be successfully picked.


\begin{figure}[ht!]
\begin{center}
\begin{overpic}[width={\iftwocolumn 3.4in \else 5in \fi},tics=5]
{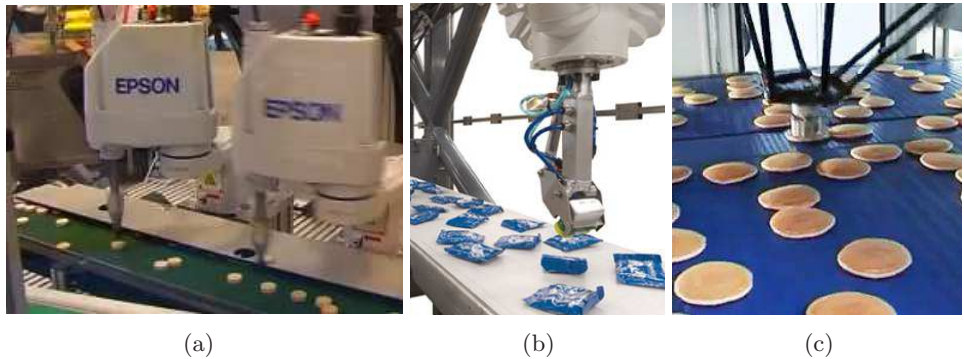}
\put(18.5, -3){{\small (a)}}
\put(54, -3){{\small (b)}}
\put(83.5, -3){{\small (c)}}
\end{overpic}
\end{center}
\caption{\label{fig:realsys} A few conveyor-based robotic \pnp systems
(a) Two Selective Compliance Assembly Robot Arm (SCARA) robots working 
on picking and placing machine parts (b) (c) Delta robots packing food 
items.}
\end{figure}
The invention and development of conveyor belt systems for material 
handling have revolutionized many industries over the years 
\cite{hounshell1985american}. With advances in computer vision 
and robotic manipulation, conveyor-based robotic \pnp solutions 
\cite{mirtich1996estimating,carlisle1994pivoting,causey1997design,
berkowitz1996designing,goldberg1997estimating} have seen rapid adoptions 
that have yielded increasing levels of automation (see Fig.~\ref{fig:realsys}).
The intrinsic goal in deploying such systems is to realize continuous 
and fast \pnp operations. Therefore, a natural algorithmic question 
to ask here is how optimal is a given solution \cite{causey1999testing,
branicky2000modeling} and how better algorithms may be designed to 
improve system throughput. 

A fairly thorough fundamental algorithmic study of conveyor-based 
robotic \pnp is carried out \cite{chalasani1996algorithms}, where 
several polynomial-time approximation algorithms are provided for 
a variety of \pnp problems. The work also pointed out that many 
such problems are at least NP-Hard \cite{lenstra1981complexity,
helvig1998moving}, given the similarity between robotic \pnp and 
the traveling salesperson problem (TSP). This and other studies, 
e.g., \cite{han2018complexity}, also link the robotic \pnp problem
to classical vehicle routing problems (VRP), which has many variations
on its own \cite{berbeglia2007static,christofides1969algorithm,
frederickson1976approximation,treleaven2013asymptotically}. We note 
that while polynomial-time approximation algorithms for \pnp have 
been proposed \cite{chalasani1996algorithms}, the algorithms optimize 
over metrics like $L_1$ and the approximately optimal 
solutions are not practical. The study also does not sufficiently consider 
robot geometry and dynamics, which are very important factors in 
real-world applications. 

When it comes to {\em practically efficient} algorithmic solutions for 
\pnp operations over a conveyer, the first proposed solutions resorted 
to a first-in first-out (FIFO) rule for prioritizing the object 
picking order \cite{li1997line,pardo1995system}. As pointed out, the FIFO 
heuristic can result in fairly sub-optimal solutions \cite{mattone2000sorting}. 
To address this, a job scheduling rule called
shortest processing time (SPT) \cite{schrage1968letter} was employed 
\cite{mattone2000sorting}. With further improvements, SPT and 
variants are shown to be consistently superior to FIFO. Since
\cite{mattone2000sorting}, research on \pnp over conveyor appears to 
have shifted to using multiple robot arms to further boost the 
throughput. Among these, non-cooperative game theory was explored 
\cite{bozma2012multirobot} whereas FIFO and SPT heuristics are 
employed \cite{daoud2014efficient}. A recent approach 
combines randomized adaptive search with Monte Carlo 
simulation \cite{huang2015robust}. 

\noindent\textbf{Contributions}. The main contributions of this 
work are two. First, after observing and analytically characterizing 
sub-optimality of existing greedy \pnp solutions, we develop a dynamic programming-based optimal finite-horizon \pnp algorithm that applies 
to arbitrary robot models for which the dynamics can be simulated. 
Within a second, our algorithm is capable of computing optimal solutions 
for over $20$ objects, which requires the exact processing of $20!$ 
possible picking sequences. With additional locality-based heuristics, 
we can compute near-optimal solutions for over $100$ objects 
in under one second. 
Second, through extensive simulation study over typical industrial 
\pnp robots (e.g., Delta and SCARA), we show that our algorithmic 
solutions are computationally efficient and outperform the existing 
state-of-the-art including \algfifo and \algspt variants by 
$10\%$ to $40\%$ in real-time settings. Such improvement is significant 
when it comes to real-world applications, where a few percentages 
of efficiency gain could provide a company a large competitive edge. 
%

\noindent\textbf{Paper Organization}. The rest of the paper is organized 
as follows. The problem setting studied in the paper is explained in 
detail in Section~\ref{sec:problem}. In Section~\ref{sec:analysis},
we demonstrate the sub-optimality of greedy methods and estimate 
the maximum potential gain via optimization. 
Sections~\ref{sec:pnptime} and~\ref{sec:algorithm} detail our algorithmic 
development, with Sections~\ref{sec:pnptime} focusing on how to quickly 
obtain optimal \pnp time for complex robots and Section~\ref{sec:algorithm}
describing how we deal with the combinatorial explosion as we seek 
optimal finite-horizon solutions. A selection of our extensive evaluation 
effort of the algorithms is presented in Section~\ref{sec:experiments}, 
demonstrating the superior real-time performance of our proposed methods. 
We then conclude with Section~\ref{sec:conclusion}.

\section{Preliminaries}\label{sec:problem}
Consider a robotic pick-and-place (\pnp) system composed of a robot 
arm and a moving conveyor belt. Such systems  
\cite{chalasani1996algorithms,schrage1968letter} are generally modeled as residing 
in a two-dimensional bounded rectangular workspace $\W \in \R^2$. 
Let the base of the robot arm be located at $(x_A, y_A = 0)$. 
We assume that the reachable area on the 
conveyor by the robot end-effector for \pnp actions is an axis-aligned 
rectangle $\W$ with the lower left coordinate being $(x_L, y_B = 0)$ and 
upper right coordinate being $(x_R, y_T)$ (see. Fig.~\ref{fig:conveyor}). 
The task for the robot is to pick up objects located within $\W$ and drop 
them off at the origin $(x_D = 0, y_D = 0)$. We assume that the rest 
position of the end-effector is also at the drop-off location. Since the 
robot will execute a large number of \pnp actions in a single run, this 
assumption has little effect on optimality. The robot is assumed to know 
all object locations within $\W$. The assumption that $y_A = 
y_B = y_D = 0$ is for convenience and has no effect on computational 
complexity and has negligible effects on solution optimality. 

\begin{figure}[ht]
\begin{center}
\begin{overpic}[width={\iftwocolumn 3.5in \else 5in \fi},tics=5]
{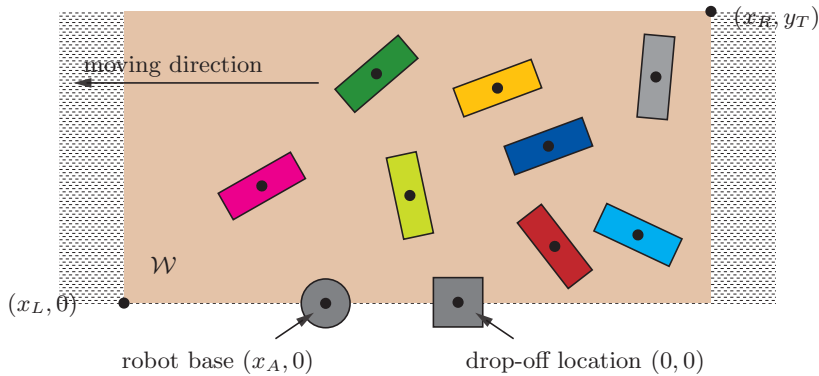}
\put(7, 3){{\small $(x_L, 0)$}}
\put(22, 7){{\small $\W$}}
\put(83, 33){{\small $(x_R, y_T)$}}
\put(19, -3){{\small robot base $(x_A, 0)$}}
\put(55, -3){{\small drop-off location $(0, 0)$}}
\put(15, 28){{\small moving direction}}
\end{overpic}
\end{center}
\caption{\label{fig:conveyor}Illustration of a conveyor workspace where 
the base of the robot arm is located at $(X_A,0)$. The end-effector picks 
up objects within a region $\W$ with a lower left corner of $(x_L, 0)$ 
and an upper right corner of $(x_R, y_T)$, and drops off objects at 
the drop-off location $(0, 0)$.}
\end{figure}

Without loss of generality, we assume that the conveyor belt moves at unit 
speed, e.g., $v_b = 1$, from the right toward the left. For the robot arm, 
we work with two types of motion models. In a {\em simplified model}, the 
end-effector of the arm is assumed to be able to extend or 
retract at a fixed speed $v_e > 1$. That is, the absolute speed of the 
end-effector along the straight line between the robot base and the 
end-effector location is $v_e$. In other words, the robot arm behaves 
like a telescoping arm. We use this 
model for the structural analysis as well as potions of the simulation 
studies.\footnote{We note that the algorithms we develop directly apply to Delta
and SCARA robots that are dominant in relevant industrial applications.
We use {\em accurate models} in Sections~\ref{sec:pnptime} for Delta and 
SCARA robots that consider both robot geometry and dynamics with bounded 
acceleration.}
Notice that in general, it is never beneficial 
for a robot to run at a lower speed in solving \pnp tasks.

For the study, it is assumed that the robot can pick up an object when 
its end-effector stops at the center of the object on the conveyor belt. 
The pickup action and the drop-off action are assumed to be instantaneous, 
i.e., they do not induce delays. We make such an assumption because the 
time involved in these actions is comparatively small in applications. 
Whereas we assume infinite acceleration and de-acceleration for the 
simplified telescoping robot model since it is a velocity based model,
as already mentioned, dynamics are carefully considered for Delta and SCARA 
robots. 
%
The overall goal is then to execute as many \pnp actions as possible
in a given amount of time. 

In developing the algorithms, we work with two object distribution models. 
Under a {\em one-shot} setting, we fix the dimensions of $\W$ and the 
number of objects $n$, and let the $n$ objects be uniformly distributed 
in a subset of $\W$. More precisely, the objects are spawned in a 
rectangular area with the same $y$ span as $\W$ and also the same 
maximum reach on the $x$ axis, i.e., both end at $x = x_R$ on the right. 
The left end of the object spawning area has an $x$ value larger than 
$x_L$ because otherwise, objects appearing close to $x_L$ may immediately 
move out of $\W$ on the conveyor, rendering it impossible to pick them. 
Under a {\em continuous} setting, which models after real application 
setups, the objects, following a some spatio-temporal distribution, appear 
at $x \ge x_R$ continuously for a period of time. For example, the 
distribution may be a Poisson process with rate $\lambda$ followed by a 
uniform distribution of $y \in (0, y_T)$. That is, as a new event is 
generated by the Poisson process, a new object is placed at $(x_R, y)$ 
where $y \in [0, y_T]$ is uniformly selected. 

Our proposed methods will be compared with greedy approaches, namely, \algfifo 
and \algspt \cite{mattone2000sorting}, which pick objects following simple 
heuristics. \algfifo follows the first-in first-out rule and always 
picks the object which enters the workspace the earliest, i.e., with the 
smallest $x$ location~\cite{mattone2000sorting}. On the other hand, 
\algspt always picks the object with the smallest \pnp time, following 
the {\em shortest processing time} rule~\cite{mattone2000sorting}. 
In addition, we add \alge which uses the Euclidean distance between an object's 
location and the drop-off location instead of 
$x$ for prioritizing. It is clear that these approaches require little 
computational effort.


\section{Analysis of the Optimal Solution Structure}\label{sec:analysis}
As greedy approaches (e.g., \algspt, \algfifo \cite{mattone2000sorting})
work with a very short horizon, it can be expected that they are generally 
sub-optimal. It appears, however, no quantitative analysis has been 
performed in the literature to study this sub-optimality. Thus, we begin 
our study with an analytical characterization on the benefit of using 
a longer horizon for optimization. 

\subsection{Non-Optimality of Greedy Strategies}\label{subsec:31}
As mentioned in the introduction, the most commonly used heuristics 
for \pnp appear to be FIFO and SPT \cite{mattone2000sorting}. Whereas 
these best-first like heuristics runs in $O(n)$ time for selecting a
single picking candidate with $n$ being the number of objects accessible 
on the conveyor, the overall solution is generally sub-optimal in terms
of efficiency over long term. We now establish this under a fairly 
general setting through examining the \pnp of two objects $o_1$ and 
$o_2$ located at $(x_1, y_1)$ and $(x_2, y_2)$, respectively 
(Fig.~\ref{fig:2objs}). For the analysis, we further assume that $x_A 
= x_D = 0$, i.e., the robot base and drop off location are the same 
(we note that the analysis that follows readily generalizes beyond 
this assumption). 

\begin{figure}[ht!]
\begin{center}
\begin{overpic}[width={\iftwocolumn 3.5in \else 5in \fi},tics=5]
{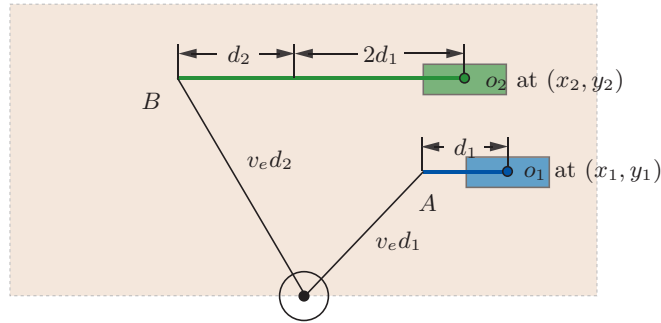}
\put(69,25.3){{\small $o_2$ at $(x_2, y_2)$}}
\put(73,15.9){{\small $o_1$ at $(x_1, y_1)$}}
\put(65.7,18.2){{\small $d_1$}}
\put(42,28){{\small $d_2$}}
\put(56.2,28){{\small $2d_1$}}
\put(57.5,8.5){{\small $v_ed_1$}}
\put(44,17){{\small $v_ed_2$}}
\put(62,12.4){{\small $A$}}
\put(33,23){{\small $B$}}
\end{overpic}
\end{center}
\caption{\label{fig:2objs}The relevant distances when object $o_1$, 
initially located at $(x_1, y_1)$, is picked up first by the robot.}
\end{figure}

Under the setup, we can compute the time it takes the end-effector to 
carry out \pnp actions on the two objects for the two possible picking 
orders (i.e., first picking the object $o_1$ or first picking the object 
$o_2$). Assuming that object $o_1$ is picked first, we can compute the 
earliest location $A$ where the end-effector can pick up the object. At 
$A$, object $o_1$ has traveled some distance $d_1$; the end-effector, 
with speed $v_e$, would have traveled a distance of $v_ed_1$. We then 
have a single unknown $d_1$ and the quadratic equation 
$
(x_1 - d_1)^2 + y_1^2 = v_e^2d_1^2,
$
with which we can solve for $d_1$. We omit the solution here, which is 
lengthy to write down. Based on $d_1$ and similar reasoning, we can compute the point 
$B$ where the end-effector can pick up the second object after dropping 
the first object at the origin. At this point, the second object would 
have traveled a distance of $2d_1 + d_2$ for some $d_2$. The equation 
for computing $d_2$ is readily obtained as 
$
(2d_1 + d_2 - x_2)^2 + y_2^2 = v_e^2d_2^2.
$

The total time required for handling both objects this way is 
$2(d_1 + d_2)$ since the conveyor runs at unit speed. We denote this 
time as $t_{12}(x_1, x_2, y_1, y_2, v_e)$. Similarly, we may compute 
the time required if $o_2$ is picked up first; denote the time 
as $t_{21}(x_1, x_2, y_1, y_2, v_e)$. It can be shown that,
setting $x_1 = x_2$, there is no general dominance between $t_{12}$
and $t_{21}$. 

\begin{proposition}\label{p:ce}
For two objects $o_1$ and $o_2$ initially located at $(x, y_1)$ and 
$(x, y_2)$, the optimal pick-and-place sequence of the objects depends 
on the horizontal offset $x$. 
\end{proposition}
\begin{proof}
We define a function $\delta t$ as 
\begin{align*}\label{eq:delta-t}
\delta t(x_1, x_2, y_1, y_2, v_e) = t_{12}(x_1, x_2, y_1, y_2, v_e) \\ - 
t_{21}(x_1, x_2, y_1, y_2, v_e).
\end{align*}
To prove the proposition, we only need to show that for some fixed $y_1, 
y_2$, and $v_e$, varying $x = x_1 = x_2$ will flip the sign of the 
function $\delta t(\cdot)$. For this purpose, we let $y_1 = 0.4, 
y_2 = 0.7$ and $v_e = 2$ and examine 
\[
f(x) = \delta t(x, x, 0.4, 0.7, 2). 
\]

Solving for $f(x) = 0$ with the restriction of $x > 0$ yields a single 
solution $x_0 \approx 0.65$. This means that when $x > x_0$ holds, it is 
more optimal to pick $o_2$ first. When $x < x_0$, it is more optimal to 
pick $o_1$ first. 
\end{proof}

Continuing from the proof of Proposition~\ref{p:ce}, if we plot $f(x)$ over 
$x \in [0.4, 1.4]$, Fig.~\ref{fig:fx} is obtained which clearly shows that 
in this case, picking $o_1$ first is only better when initial $x$ is 
less than $0.65$. It also shows that $\delta t$ can have relatively 
large positive and negative values, meaning that the conclusion of  
Proposition~\ref{p:ce} holds for proper $x_1 \ne x_2$ and also for $x_A 
\ne x_D$ when other conditions are proper. 
\begin{figure}[ht]
\begin{center}
\begin{overpic}[width={\iftwocolumn 3.5in \else 5in \fi},tics=5]
{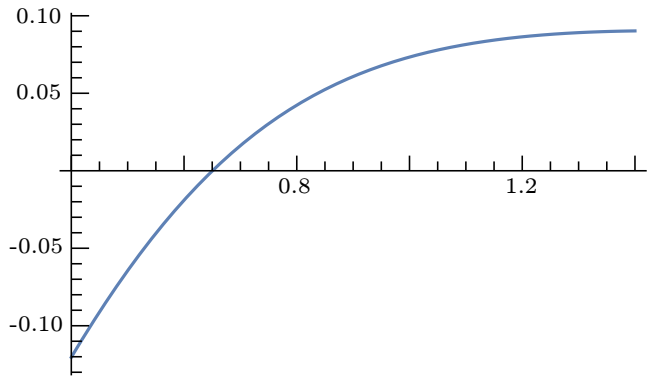}
\end{overpic}
\end{center}
\caption{\label{fig:fx}Plot of $f(x) = \delta t(x, x, 0.4, 0.7, 2)$ 
for $x \in [0.4,1.4]$.}
\end{figure}

{\em Practical Concerns}. As one might expect, as indicated by the SPT 
rule, it is generally better if the closer object is picked first. In 
the example from the proof of Proposition~\ref{p:ce}, the maximum value of 
$f(x)$ is reached when $x \approx 1.45$, which yields $f(1.45) \approx 
0.09$. The value of $t_{21}$ in this case is approximately $0.77$ 
(therefore, $t_{12} \approx 0.86$). That is, picking $o_2$ first in 
this case will lead to at most an optimality loss of $0.09/0.77 
\approx 12\%$. On the other hand, always picking $o_2$ first can lead 
to an optimality loss of about $40\%$ when $x \approx -0.1$. Nevertheless, 
a $12\%$ of optimality loss is very significant and should 
be avoided in practice whenever possible. Lastly, since the analysis 
is based on $x_1 = x_2$, it directly applies to the FIFO setting as well,
where the more sub-optimal choice of picking $o_2$ first can happen if
$x_1$ is just slightly larger than $x_2$. 

\subsection{Structure of Optimal \pnp Solutions}
The analysis and resulting observation from Section~\ref{subsec:31} 
leads us to develop optimal \pnp solutions via exhaustive search 
(see Section~\ref{subsec:41}). 
Through running the exhaustive search algorithm, we computed optimal 
solutions under various configurations and studied the distribution 
pattern of the optimal picking sequences. A typical outcome under 
the one-shot setting is illustrated in Fig.~\ref{fig:opt-distr}. In 
the figure, the small discs correspond to the initial locations of 
objects for $100$ problems with $n = 10$ each, with $2 \le x \le 8$,
$0 \le y \le 3$, and $v_e = 5$. After computing the optimal solution 
for each of the $100$ problems, we color the object that is picked 
first (out of the $10$ objects in a problem instance) dark red and 
the last picked object dark blue. The colors for the other eight 
objects are interpolated between these two. The majority of objects 
are picked before their $x$ coordinates fall below $x = -2$.  
\begin{figure}[ht]
\begin{center}
\begin{overpic}[width={\iftwocolumn 3.4in \else 4in \fi},tics=5]
{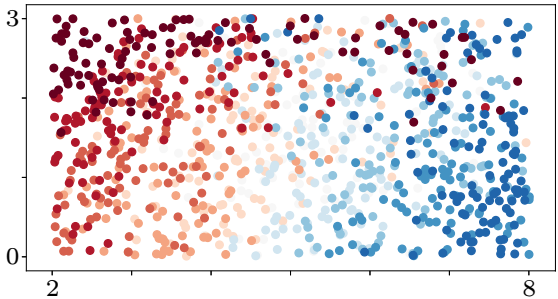}
\end{overpic}
\end{center}
\caption{\label{fig:opt-distr} Illustration of picking orders of $10$ 
objects in $100$ optimal solutions.}
\end{figure}

From the figure, we make the observation that the first few objects that 
get picked are concentrated toward the left, though a few are also from 
the far right. Generally, however, they all have relatively large $y$ 
coordinates. The last few objects, on the other hand, fall more on the 
far right and are not concentrated in terms of the $y$ coordinate. The 
objects that are picked in the middle in an optimal sequence tend to 
fall in the middle, which more or less is as expected.

\section{Computation of Shortest \pnp Time}\label{sec:pnptime}
A significant challenge in the design and implementation of object 
picking sequence selection algorithms is how to deal with the geometry 
and dynamics of the robots (see, e.g., \cite{carp2003dynamic}) that are 
involved. We encapsulate the complexity caused by robot geometry and 
dynamics in a routine, \getpnp, that returns the best available \pnp 
time for a given robot model and the initial location of the moving 
object to be picked up. This is achieved through a two-step process. 
First, a principled method is designed for estimating a single optimal 
\pnp time. Second, we build a table of pre-computed \pnp times to enable 
real-time look-up in practice. 

\subsection{Computing Shortest \pnp Time for Simple Robots}
If the robot has trivial dynamics (note that this is a BIG if that
almost never happens in practice), it may be possible for \getpnp to 
compute the \pnp time directly and analytically. In the case of the simplified 
telescoping robot, we may do so via solving the quadratic equation 
\[(\sqrt{x_A^2 + y_A^2} \pm v_e t)^2 = [(x - v_b t) - x_A]^2 + (y - y_A)^2.\]
The sign of $v_e t$ depends on whether the arm extends or retracts, which 
directly correlates to whether the object's current location $(x, y)$ is 
in the circle centered at $(x_A, y_A)$ with radius $\sqrt{x_A^2 + y_A^2}$ 
(see Fig.~\ref{fig:getpnp}). In this case, the arm extension and retraction
take the same amount of time. 

\begin{figure}[ht]
    \begin{center}
    \begin{overpic}[width={\iftwocolumn 3.3in \else 4in \fi},tics=5]
    {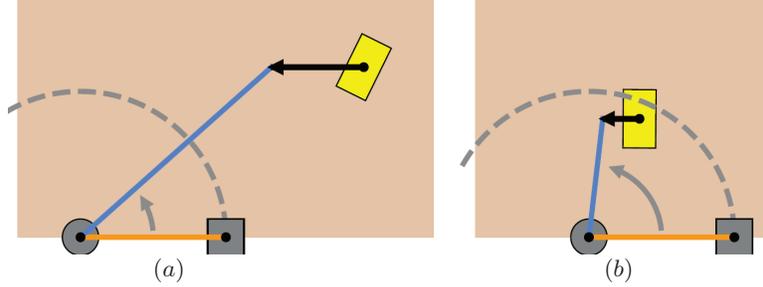}
    \put(19, -3){{\small $(a)$}}
    \put(78, -3){{\small $(b)$}}
    \end{overpic}
    \end{center}
    \caption{\label{fig:getpnp} Scenarios when robot $(a)$ extends and $(b)$ 
    contracts the arm. The orange and blue lines illustrate the drop-off and 
    pick-up arm poses, respectively. the arced arrows show the rotation 
    movements of the arm.}
\end{figure}

\subsection{Computing Shortest \pnp Time for Complex Robots}
Computing optimal \pnp time is hard in general as most robots have 
complex, interacting geometric constraints and  physical 
constraints including robot kinematics, speed/acceleration limits, and 
so on. For a given robot model, e.g., SCARA, we first need a method for 
computing the optimal (shortest) time it takes for the end-effector to 
reach a point $(x, y)$ within the robot's workspace and then the optimal 
time for the end-effector to return to the drop-off location. The sum of 
the two times is the optimal total \pnp time. In practice, optimal \pnp 
times are estimated \cite{carp2003dynamic}. 

Our implementation of the shortest \pnp time computation for a single 
2D point is as follows. First, based on the robot's geometric structure, 
we compute the joint angles of the robot (two for SCARA and three for Delta 
\cite{clavel1988fast,carp2003dynamic}) at the initial (drop-off) 
end-effector location and the target pick-up location $(x, y)$. Then, 
we invoke the Reflexxes Motion Library \cite{zsombor2004descriptive} to 
obtain an estimated shortest transition time from the drop-off pose to 
the pick-up pose and also the shortest transition time from the pick-up 
pose to the drop-off pose, which may be different. 
We denote this time as $t(x, y)$, from which we can readily obtain  
$(x + v_bt(x,y), y)$ as the location where the object is before the 
end-effector starts the \pnp operation. 

Because each computation of $t(x, y)$ can be relatively time consuming 
(easier for SCARA and slightly more involved for Delta with 3 degrees of 
freedom), the procedure cannot directly be used for real-time robot 
operations. Instead, we build a table of pre-computed \pnp times at a 
given resolution (in this paper, a $100 \times 100$ discretization is 
used, with interpolation), with which \getpnp can then be realized 
extremely efficiently with very high precision. 

\subsection{Visualizing Typical \pnp Time Profiles}
The \getpnp subroutine can be readily adapted to work with other 
robots. That is, \getpnp is an abstraction layer that isolates the object
picking sequence selection from physical robot models. It is clear that
different robots can have significantly different \pnp time structure. 
For the three robots that are examined in this paper, their \pnp time 
profiles are shown in Fig.~\ref{fig:profiles}. We note that the (rotated) profiles 
are slightly truncated at the bottom. Similarly, there are some values
missing at the top of the figures; this is because objects initially  
located in these areas will exit workspace before the arm can reach them.

\begin{figure}[ht]
    \begin{center}
    \begin{overpic}[width={\iftwocolumn 1.1in \else 1.5in \fi},tics=5]
    {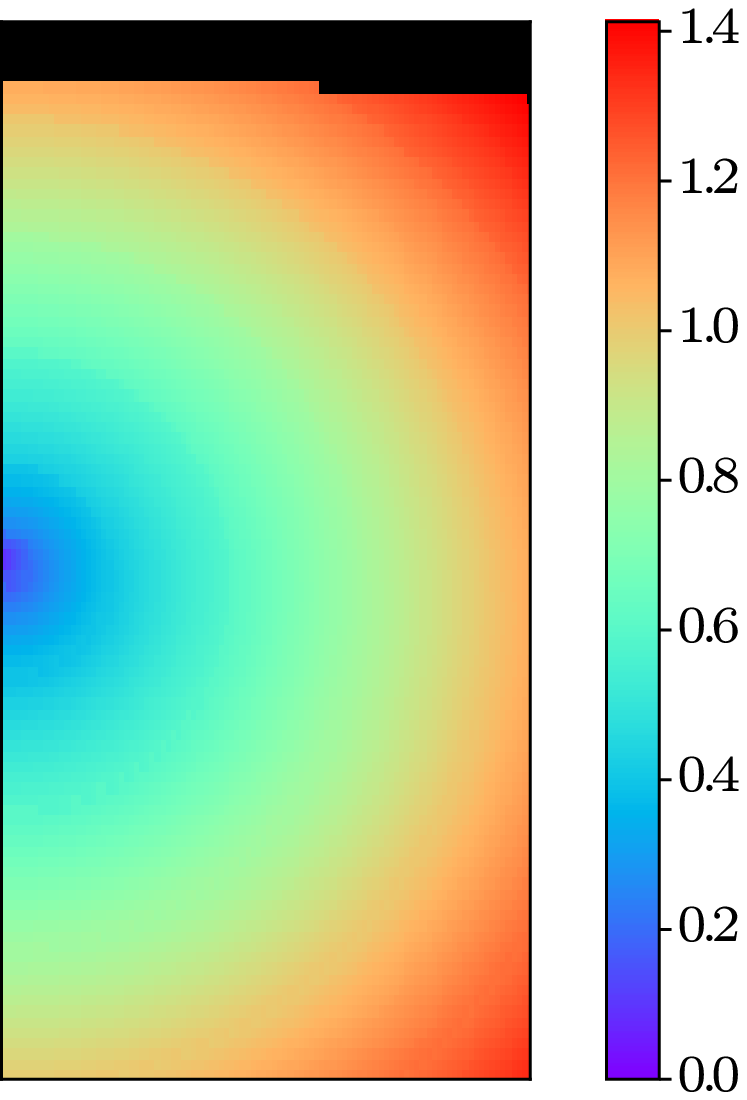}
		\put(27, -7){{\footnotesize (a)}}
    \end{overpic}
    \begin{overpic}[width={\iftwocolumn 1.1in \else 1.5in \fi},tics=5]
    {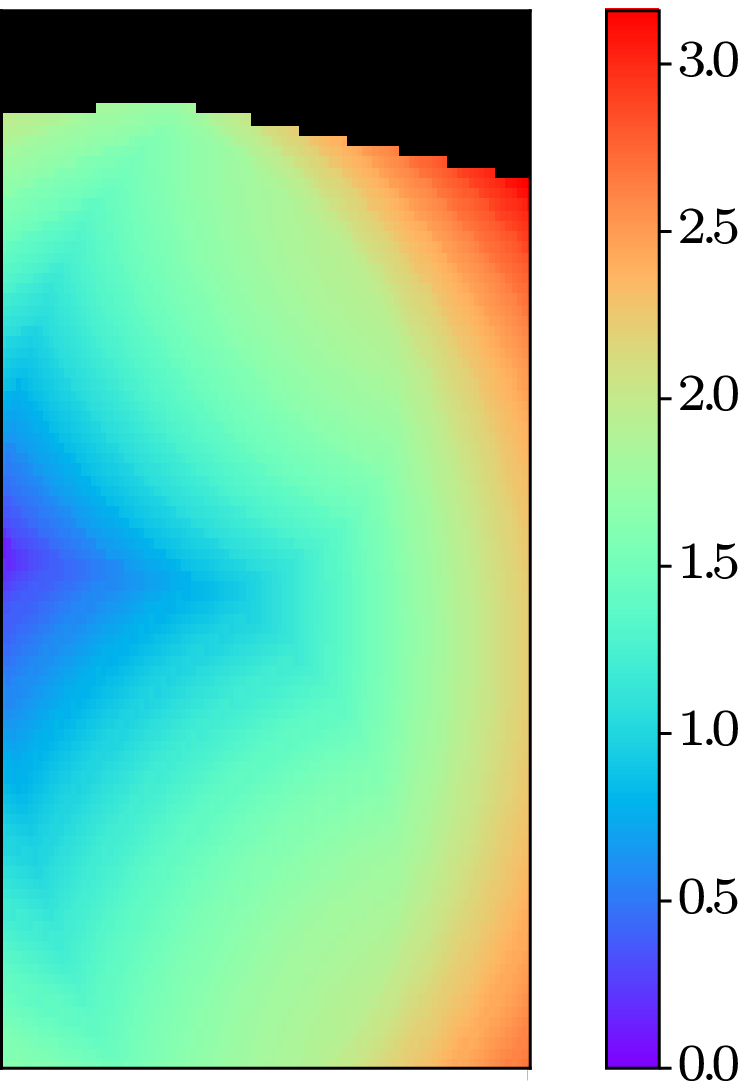}
		\put(27, -7){{\footnotesize (b)}}
    \end{overpic}
    \begin{overpic}[width={\iftwocolumn 1.1in \else 1.5in \fi},tics=5]
    {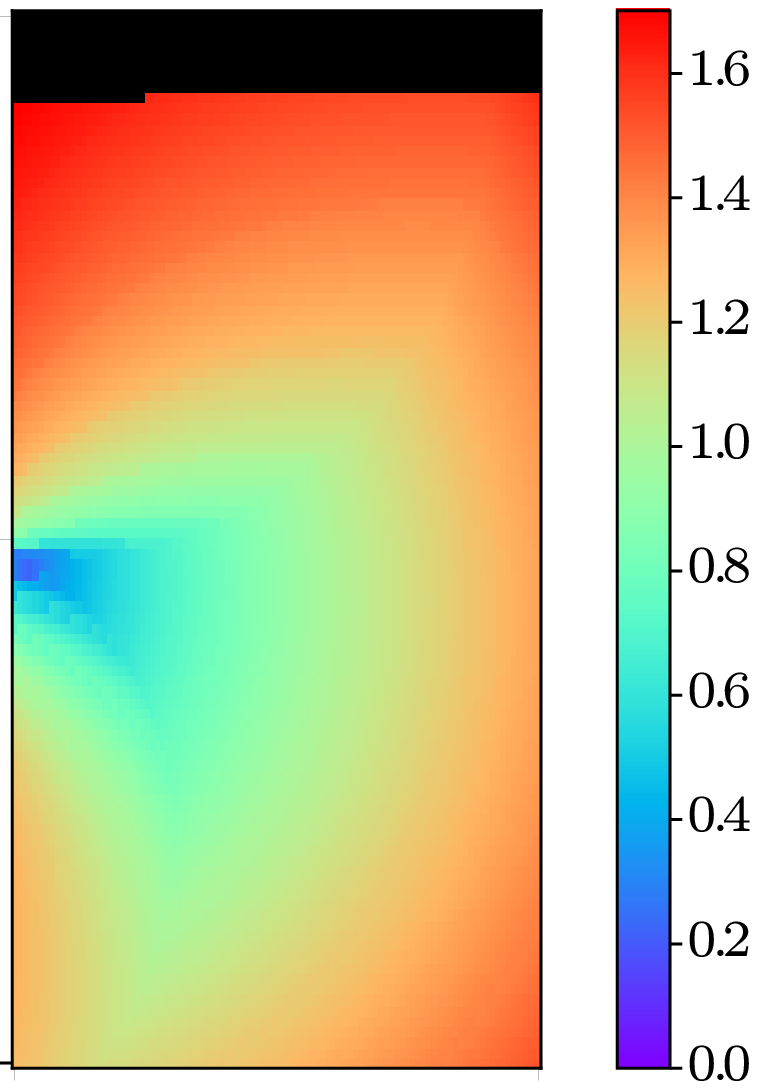}
		\put(27, -7){{\footnotesize (c)}}
    \end{overpic}
    \end{center}
    \caption{\label{fig:profiles} The (relative) time profile for \pnp
		operations for difference arms. The workspace is rotated $90$ degrees
		clockwise. (a) The simplified telescoping robot. (b) The Delta robot.
		(c) The SCARA robot.}
\end{figure}

\section{Exact And Approximate Algorithms for Selecting The Best Picking Sequence}\label{sec:algorithm}
The main algorithm developed in this work is an {\em exhaustive-search}-based 
method which checks all possible object picking sequences to find 
the optimal one. In addition, a {\em local-augmentation}-based method 
is developed to further boost computational efficiency.

\subsection{Exhaustive Search Methods: \algperm and \algdp}\label{subsec:41}
With the \getpnp routine, a baseline exhaustive search routine is 
straightforward to obtain. We call such a routine \algperm, which computes 
the optimal object picking sequence for a given horizon (i.e., number of objects
examined at a time). Then, {\em dynamic programming} is applied to speed up 
\algperm, yielding the routine \algdp, which is significantly faster yet 
without any loss of optimality. A similar application of dynamic programming 
in the robotics for a different application can be found in \cite{TanYu19ISRR}.

\subsubsection{\algperm} 
As it is shown in Alg.~\ref{alg:perm}, \algperm iterates through all 
permutations of objects (line~\ref{alg:perm:all}) and finds the picking 
sequence with the minimum execution time 
(lines~\ref{alg:perm:calculate}--\ref{alg:perm:compare}). The computation 
time of \algperm is $O(n!n)$, since there are $n!$ permutations to check and 
for each permutation, the algorithm calls \getpnp $n$ times. 

\begin{algorithm}
    \DontPrintSemicolon
    \KwIn{objects' initial location $(x_1, y_1), \dots, (x_n, y_n)$}
    \KwOut{$S^*$: a time-optimal \pnp sequence}
    $t^* \gets \infty, S^* \gets \text{none}$\;
    \For{$P \in $ \textsc{AllPermutations($\{1, \dots, n\}$)}}{ \label{alg:perm:all}
        $t \gets 0$\; \label{alg:perm:calculate}
        \lFor{$i \in P$}{$t \gets t$ $+$ \getpnp($x_i - v_b t, y_i$)}
        \lIf{$t < t^*$}{$t^* \gets t, S^* \gets P$} \label{alg:perm:compare}
    }
    \Return $S^*$\;
    \caption{\algperm}\label{alg:perm}
\end{algorithm}

\subsubsection{\algdp} 
Clearly, \algperm contains redundant computation. For example, for four 
objects, the time for first picking up objects $(1, 2)$ is calculated twice, 
during the computation for sequences $(1, 2, 3, 4)$ and $(1, 2, 4, 3)$. 
To avoid such redundant calculations, we propose \algdp, a dynamic programming 
algorithm similar to that in~\cite{held1962dynamic}. 
The pseudo code for \algdp is provided in Alg.~\ref{alg:dp}. In
line~\ref{alg:dp:init}, two datasets are initialized: $S$ which
will contain the time-optimal picking sequences of all the $2^n$ subsets of objects, and $T$ 
which will contain the associated time costs. An $n$-step iterative process 
starts from line~\ref{alg:dp:iteration}. In line~\ref{alg:dp:update}, the 
algorithm updates $T$ amd $S$ for each $k$-combination $U$. The update 
process iterates through all objects $i \in U$, finds the one that minimizes 
the execution time when picked last:
\sh{ equation too long, will fix it later.}
\begin{multline*}
    T[U] = \min_{i \in U} \{\,T[U \backslash \{i\}] + \\ \getpnp(x_i - v_b \, T[U \backslash \{i\}], y_i)\,\}.
\end{multline*}
During this process, $S[U]$ is also updated accordingly to store the 
subsets' optimal picking sequence. Finally, in line~\ref{alg:dp:return}, a time-optimal \pnp sequence 
of all $n$ objects is returned.

\begin{algorithm}
    \DontPrintSemicolon
    \KwIn{objects' initial location $(x_1, y_1), \dots, (x_n, y_n)$}
    \KwOut{a time-optimal \pnp sequence}
    $T = \{\varnothing : 0\}, S = \{\varnothing : ()\}$\; \label{alg:dp:init}
    \For{$1 \leq k \leq n$}{ \label{alg:dp:iteration}
        \For{$U \gets$ \textsc{AllCombinations($\{1, \dots, n\}, k$)}}{ \label{alg:dp:combinations}
            \textsc{Update($T$, $S$, $U$)} \label{alg:dp:update}
        }
    }
    \Return $S[\{1, \dots, n\}]$\; \label{alg:dp:return}
    \caption{\algdp}\label{alg:dp}
\end{algorithm}
\jy{It is usually common to have a back trace to get the optimal solution.
Make sure the pseudo code is correct. I will leave this to you guys.}

\begin{proposition}
    \algdp finds the optimal \pnp sequence.
\end{proposition}
\begin{proof}
    Since it is trivial that $S$ contains the optimal picking sequence of the
    $0$-combination of the objects, 
    it suffices to show that given the optimal sequences of all 
    $(k - 1)$-combinations, then the function 
    \textsc{Update}($T, S, U$) in line~\ref{alg:dp:update} calculates the optimal 
    sequences of all $k$-combinations. 

    Given $U$ as an arbitrary $k$-combination, its optimal picking sequence must 
    also be optimal when picking the first $k - 1$ objects in this sequence. 
    The update process of $T[U]$ checks all candidate sequences with the first 
    $k - 1$ objects picked in a time-optimal manner.
\end{proof}

\algdp runs in $O(2^n n)$: there are $O(2^n)$ object subsets, and processing 
a combination $U$ calls the routine \getpnp $|U|$ times, taking $O(n)$ time. 
Since $2^n \ll n!$ for large $n$, \algdp is much faster than \algperm. 

\subsection{A Local Augmentation Method: \alglocal} 

While working on \algdp, we attempted many heuristics to further boost 
its efficiency. Here, we report a particularly effective method 
that appears to achieve optimality close to \algdp but scales much better. 
We call this local augmentation-based method \alglocal, which uses \algdp 
as a subroutine. The pseudocode of \alglocal is provided in 
Alg.~\ref{alg:local}. In line~\ref{alg:local:init}, \alglocal starts with 
an initial picking sequence $S$, which can be selected 
in many different ways, for example, using FIFO. Then, 
line~\ref{alg:local:m1}-\ref{alg:local:m2} repeatedly call \algdp over 
sub-sequences of $S$ to reduce the execution time. Specifically, the 
algorithm has two parameters $m_1$ and $m_2$. The main loop is repeated 
$m_1$ times, and for each iteration, we call \algdp over the 
$k\textsuperscript{th}$ to $(k + m_2)\textsuperscript{th}$ elements of $S$ 
for $1 \leq k \leq n - m_2$.

The computation time for \alglocal is $O(m_1 n 2^{m_2} m_2)$. 
In our implementation, we found that initializing $S$ using \algfifo and assigning 
$m_1 = n, m_2 = 9$ produce solutions that are often indistinguishable from 
these computed by \algdp: in all test cases, the average performance difference 
between \alglocal and \algdp never exceeds $0.05\%$.

\begin{algorithm}
    \DontPrintSemicolon
    \KwIn{objects' initial location $(x_1, y_1), \dots, (x_n, y_n)$}
    \KwOut{a near-optimal \pnp sequence}
    $S \gets$ \textsc{GetInitialPickingSequence()}\; \label{alg:local:init}
    \For{$m_1$ times}{ \label{alg:local:m1}
        $t \gets 0$\;
        \For{$1 \leq k \leq n - m_2$}{
            $O \gets \varnothing$\; \label{alg:perm:calculate}
            \For{$i \in S[k:k + m_2]$}{$O \gets O \cup \{(x_i - v_b t, y_i)\}$}
            $S[k:k + m_2] \gets$ \algdp($O$)\;
            $t \gets t$ $+$ \getpnp($x_{S[k]} - v_b t, y_{S[k]}$)\;
            } \label{alg:local:m2}
    }
    \Return $S$\; \label{alg:local:return}
    \caption{\alglocal}\label{alg:local}
\end{algorithm}

\begin{remark}[Adapting to a conveyor setting]
Since it is expected that a conveyor will run without stopping for 
extended periods of time, for the continuous setting, \algdp or \alglocal 
are invoked repeatedly with real-time locations of all the pick-able 
objects in the workspace. 
\end{remark}

\section{Experimental Studies}\label{sec:experiments}
\setlength\tabcolsep{6pt} 
\def\arraystretch{1.2} 

We performed an extensive evaluation of the newly developed algorithms. 
In this section, we present a small subset of the evaluation that is 
most representative. 
In Section~\ref{subsec:time}, we measure the computation time of the 
algorithms under the one-shot setting and conclusively show that our 
algorithms are sufficiently fast for industrial applications. 
In Section~\ref{subsec:one-simple}, we focus on the one-shot setting
and check how much execution time savings are possible. Selected results 
demonstrate that the projected efficiency gain of our algorithms are 
very significant across different robot models. 
Finally, in Section~\ref{subsec:continuous}, we evaluate the performance 
of the SCARA robot in realistic continuous conveyor settings under two 
different object arrival distribution models. Again, our proposed 
algorithm shows a clear lead. 

All algorithms are implemented in 
C\nolinebreak[4]\hspace{-.05em}\raisebox{.4ex}{\small\bf ++}, and all 
experiments are executed on an Intel\textsuperscript{\textregistered} 
Xeon\textsuperscript{\textregistered} CPU at 3.0GHz. 

\subsection{Computational Efficiency} \label{subsec:time}
Fig.~\ref{fig:computation-time} shows the computation time of our original 
algorithms versus the number of objects $n$. 
With dynamic programming, we can solve much larger problem instances 
(near-)optimally: in around one second, \algperm, \algdp, \alglocal can 
solve one-shot problems with $10$, $22$ and $100$ objects, respectively. 
We note that, while the test is done over the simplified robot model, 
similar performance is observed for Delta and SCARA robots. 
The shape of \alglocal is due to the choice of parameters (i.e. $m_2 = 9$).

\begin{figure}[ht]
    \begin{center}
    \includegraphics[width={\iftwocolumn 3.45in \else 4in \fi},tics=5]{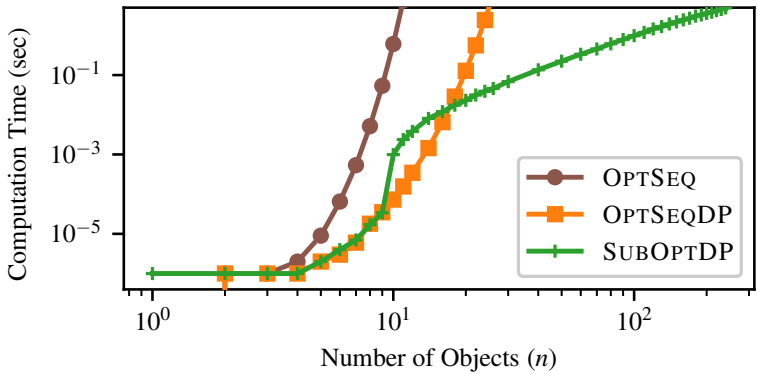}
    \end{center}
    \caption{\label{fig:computation-time} Average computation time of our 
    algorithms versus the number of objects $n$. In all experiments, 
    the error is within $\pm 2 \%$.}
\end{figure}

Our observation indicates that the active workspace in a conveyor \pnp 
system contains a few to low tens of objects. From the figure, we observe 
that both \algdp and \alglocal can complete a single sequence computation 
for ten objects within $10^{-4}$ seconds and fifteen objects with $10^{-2}$ 
seconds. Because, Delta and SCARA-based \pnp systems generally do not pick 
more than a single digit number of objects per second, \algdp and \alglocal 
impose negligible time overhead. As such, they are sufficiently fast for 
the target industrial applications. 

\subsection{\pnp Performance under One-Shot Setting} \label{subsec:one-simple}
Having shown that \algdp and \alglocal are sufficiently fast, next, we 
present their performance in one-shot settings where only a single batch 
of objects are handled. In a typical setting, we let $x_L = -5, x_R = 5, 
y_B = 0$, and $y_T = 5$, i.e., the workspace is a $10\times 5$ rectangle. 
All objects are initially uniformly randomly placed between 
$x = 3$ and $x=5$. Robots are configured so that they can reach anywhere 
within the workspace but are forbidden to reach outside. We first present 
Fig.~\ref{fig:performance-static}, which illustrates the relative total 
\pnp time of different algorithms. For results in the top two figures, 
since the number of objects is comparatively small, a faster relative 
conveyor belt speed is used ($5\times$ of that used in the bottom 
two figures). The parameters of the robots are set so that all algorithms 
can successfully pick all objects (for each figure, only a single set of 
robot parameters are used). Our selection of the robot model is somewhat 
arbitrary because we observe some but no substantial difference between 
different robot models. Therefore, we decided to select a diverse set of 
results given the limited available space. 
\begin{figure}[ht]
    \begin{center}
    \includegraphics[width={\iftwocolumn 3.45in \else 4in \fi},tics=5]{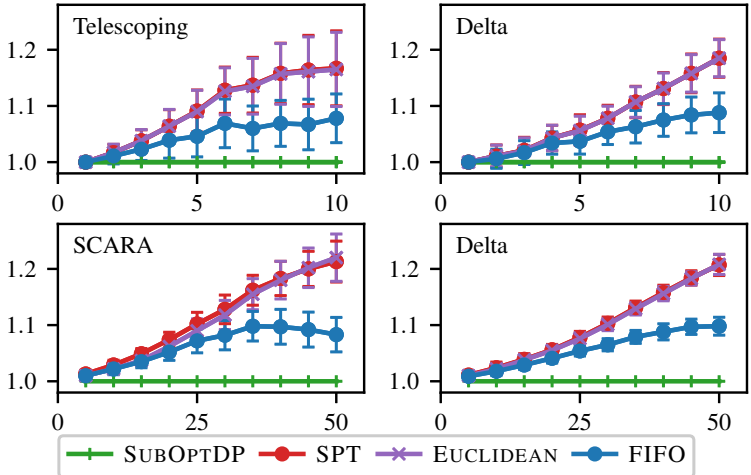}
    \end{center}
    \caption{\label{fig:performance-static} Execution time ratios of 
		various algorithms as compared with \alglocal. The $y$-axes are the
		ratio and the $x$-axes are the number of objects used in a given 
		run. Each data point is an average over 100 randomly generated 
		instances. Standard deviations are plotted as error bars.}
\end{figure}

From the figure, we observe that \alglocal (and therefore \algdp, which is 
at least as fast as \alglocal) yields significant savings in \pnp execution 
time. For example, for a SCARA robot, if we expect the workspace to have 
about ten objects at a time, then \algfifo, \alge, and \algspt are expected 
to spend around $10\%$ to $20\%$ more \pnp execution time as compared to 
our proposed solutions. 

We also observe that there does not seem to be an upper bound on the 
ratios as the number of objects increases. Though it appears that the 
ratio for \algfifo is tapering off in some of the figures, adding more 
objects shows that these ratios will 
eventually grow again. Though this may not be highly relevant in practice 
as the the number of objects is not likely to exceed a few tens, the 
phenomenon is structurally interesting. Our interpretation is that the 
behavior is perhaps caused by the optimal object picking sequence 
problem is similar in structure as hard TSPs \cite{gonzalez1976open} 
where polynomial time constant factor approximations are provably impossible. 

We also evaluated the impact of different workspace settings with the 
result given in Fig.~\ref{fig:performance-static-w} for the Delta 
robot model (again, other robot models yield similar results). We note 
that the execution time here is relative but provides a meaningful 
comparison between different workspace settings. 10 objects are used 
for each run, which are randomly allocated between $x = 3$ and $x =5$. 
We conclude that the impact of workspace 
appears to be small. 
\begin{figure}[ht]
    \begin{center}
    \includegraphics[width={\iftwocolumn 3.45in \else 4in \fi},tics=5]{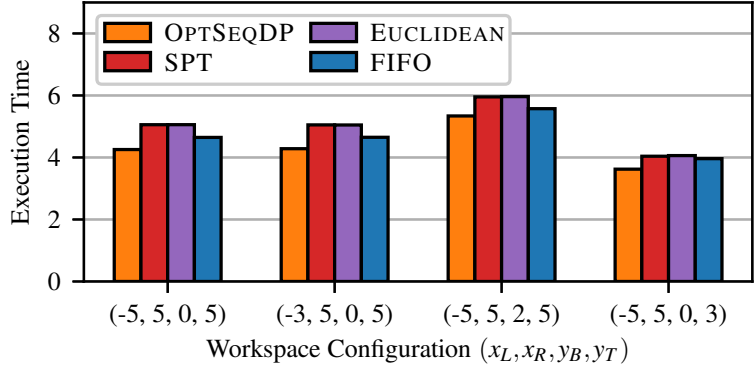}
    \end{center}
    \caption{\label{fig:performance-static-w} The execution time for 
		\pnp operations on 10 randomly placed objects. Each bar is obtained
		as an average over 100 runs.}
\end{figure}

\subsection{\pnp Performance under Continuous Setting} \label{subsec:continuous}
After the one-shot setting, we examined a more realistic setting and 
evaluated how the algorithms perform when the conveyor runs for an 
extended period of time. For this setting, we fixed the conveyor speed 
and robot parameters so that there are generally a few objects on the 
conveyor within the workspace, mimicking practical settings. We attempted 
two distributions with which objects are placed on the conveyor: Poisson 
and uniform. For all experiments, we sample $n = 10000$ objects and used a 
workspace with $x_L = -5, x_R = 5, y_B = 0$, and $y_T = 5$. The SCARA 
robot model is used here for two reasons: {\em (i)} SCARA is the most widely 
used industrial \pnp robot and {\em (ii)} the performance of SCARA is similar
to Delta and the simplified telescoping model. 

For the Poisson setting, a Poisson process with parameter 
$\lambda > 0$ is started at time $t_0 = 0$. Each time an event is triggered
by the process at time $t \ge t_0$ (including at $t = 0$), we sample the 
uniformly sample $(y_B, y_T)$ to get a $y$ value. An object is then placed 
at $(x_R, y)$ at time $t$. For the uniform setting, we sample $n$ points in 
the unit square and the scale the unit square to have the same height as the 
workspace. We then adjust the length of the unit square to simulate how 
densely the objects are placed on the conveyor belt. At $t_0 = 0$, the 
left side of the scaled unit square is aligned with the line segment 
between $(x_R, y_B)$ and $(x_R, y_T)$. 

In each experiment, we record the total number of objects that can be 
successfully picked up before some leave the left side of the workspace 
on the conveyor. The results are then scaled as ratios divided by 
$n = 10000$ and the data is visualized in Fig.~\ref{fig:continuous-poisson} 
and Fig.~\ref{fig:continuous-uniform}. 

\begin{figure}[ht]
    \begin{center}
    \includegraphics[width={\iftwocolumn 3.45in \else 4in \fi},tics=5]{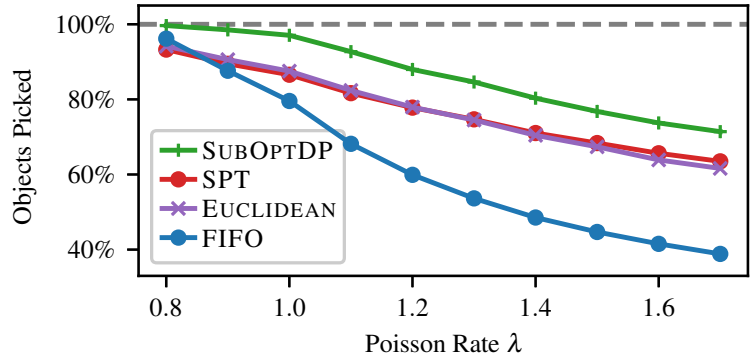}
    \end{center}
    \caption{\label{fig:continuous-poisson} Percentage of objects picked up
		out of $10000$ using different algorithms. The object distribution is 
		generated by 1D uniform distribution driven by a Poisson process with 
		different Poisson rate $\lambda$. We mention that the absolute value of 
		$\lambda$ does not bear much significance.}
\end{figure}

Looking at Fig.~\ref{fig:continuous-poisson}, except when the Poisson rate 
$\lambda$ is sufficiently low so that almost all objects can be picked up 
using any method, \alglocal maintain a lead of around $10\%$ as 
compared with \algspt and \alge. The lead over $\algfifo$ is as large as 
$40\%$. We also point out that, unlike the one-shot case, \algfifo generally 
performs the worst, though it works better than \algspt and \alge initially.
This is as expected since \algfifo does the least amount of optimization. 
In the one-shot setting, requiring that all objects can be picked up benefited 
\algfifo since it never let an object travel too far to the left side of 
the workspace. On the other hand, when there are too many objects, \algfifo
suffer since it attempts to catch objects that can be very far on the 
left side of the workspace, inducing penalty. 

\begin{figure}[ht]
    \begin{center}
    \includegraphics[width={\iftwocolumn 3.45in \else 4in \fi},tics=5]{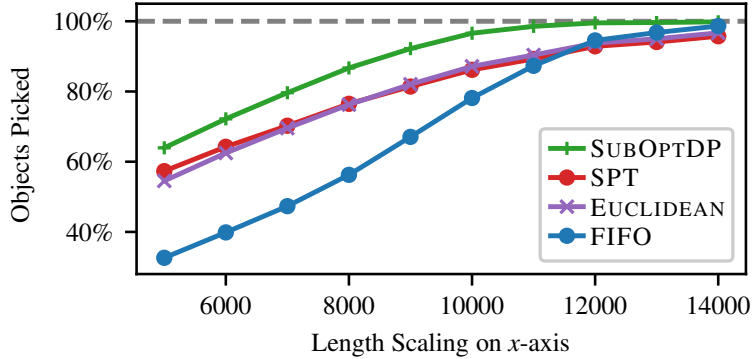}
    \end{center}
    \caption{\label{fig:continuous-uniform} Percent of objects picked up
		out of $10000$ using different algorithms. The object distribution is 
		generated by uniform distribution over the unit square. The height is 
		then scaled to $5$ and length is scaled as indicated in the figure, from 
		$5000$-$14000$.}
\end{figure}

For the uniform setting, e.g., Fig.~\ref{fig:continuous-uniform}, we 
observe a similar outcome as the Poisson setting. The figure looks 
different because larger scaling of the length of the unit square 
means sparser object placement and thus an easier setting, whereas 
larger $\lambda$ values suggesting denser object placement. 

From the experiments, we conclude that \algdp and \alglocal provide 
significant performance gain as compared with \algfifo, \algspt, and 
\alge. 

An illustration video of the proposed algorithms is provided at 
\url{https://youtu.be/bIomJzjKXyc}.

\section{Conclusion and Discussions}\label{sec:conclusion}
\vspace{-2mm}
Seeing that the state-of-the-art for robotic \pnp on a moving conveyor 
from the literature generally take a greedy approach in deciding the 
object picking order, we set out to explore how sub-optimal such 
approaches may be and ways to improve them. Using a simplified telescoping 
robot model, we show that greedy methods may lead to the loss of
$12\%$ to $40\%$ execution time efficiency. To address the shortcomings of greedy 
methods, we propose the exhaustive \algperm algorithm that computes optimal
picking sequences for a finite look-ahead horizon. The algorithm is then
further enhanced with dynamic programming and heuristic techniques to yield
\algdp and \alglocal that are sufficiently efficient for practical 
setting. In doing so, we also establish a principled method for computing 
shortest \pnp time profiles for complex Delta and SCARA robots, which 
have highly involved geometric and dynamic constraints. 

Extensive simulation-based experimentation indicates that \algdp and 
\alglocal fully realize the goal we set out to achieve, as reflected 
in three aspects: {\em (i)} both algorithms are  
computationally efficient for real-time \pnp applications, {\em (ii)}
in the one-shot setting, our algorithms deliver up to $20\%$ saving in 
execution time as compared to \algfifo, \alge, and \algspt, 
and {\em (iii)} in realistic \pnp operations on continuously running 
conveyors, our algorithms show $10$-$40\%$ advantage in terms of the 
number of picked objects. The magnitude of the efficiency 
gain has significant practical implications; a few percentage of 
difference in efficiency can separate success from failure. 
We observe no substantial difference in performance of our algorithms 
as we switch between robot models (i.e., telescoping, 
Delta, and SCARA). We conclude that \algdp and \alglocal could 
potentially make sizable impact to industrial \pnp systems. 

In future work, we would like to further explore two directions. 
First, a natural next step is to develop algorithms for the
collaboration among multiple robots to further enhance overall system 
throughput, which appears to require more carefully object selection across 
multiple robots. As a second direction, we are in the process of 
integrating our algorithms on some real robot hardware, with the hope 
of bringing our methods one step closer to real-world applications. 

\jy{It could be fairly interesting to explore a bit the case of multiple 
robots. It seems to have a reasonable amount of workload and could make 
a great addition to generate a journal publication.}

\vspace{-2mm}
{\small
\bibliographystyle{IEEEtran}
\bibliography{../bib/all}

\begin{thebibliography}{10}
\providecommand{\url}[1]{#1}
\csname url@samestyle\endcsname
\providecommand{\newblock}{\relax}
\providecommand{\bibinfo}[2]{#2}
\providecommand{\BIBentrySTDinterwordspacing}{\spaceskip=0pt\relax}
\providecommand{\BIBentryALTinterwordstretchfactor}{4}
\providecommand{\BIBentryALTinterwordspacing}{\spaceskip=\fontdimen2\font plus
\BIBentryALTinterwordstretchfactor\fontdimen3\font minus
  \fontdimen4\font\relax}
\providecommand{\BIBforeignlanguage}[2]{{%
\expandafter\ifx\csname l@#1\endcsname\relax
\typeout{** WARNING: IEEEtran.bst: No hyphenation pattern has been}%
\typeout{** loaded for the language `#1'. Using the pattern for}%
\typeout{** the default language instead.}%
\else
\language=\csname l@#1\endcsname
\fi
#2}}
\providecommand{\BIBdecl}{\relax}
\BIBdecl

\bibitem{hounshell1985american}
D.~Hounshell, \emph{From the American system to mass production, 1800-1932: The
  development of manufacturing technology in the United States}.\hskip 1em plus
  0.5em minus 0.4em\relax JHU Press, 1985, no.~4.

\bibitem{mirtich1996estimating}
B.~Mirtich, Y.~Zhuang, K.~Goldberg, J.~Craig, R.~Zanutta, B.~Carlisle, and
  J.~Canny, ``Estimating pose statistics for robotic part feeders,'' in
  \emph{Proceedings of IEEE international conference on robotics and
  automation}, vol.~2.\hskip 1em plus 0.5em minus 0.4em\relax IEEE, 1996, pp.
  1140--1146.

\bibitem{carlisle1994pivoting}
B.~Carlisle, K.~Goldberg, A.~Rao, and J.~Wiegley, ``A pivoting gripper for
  feeding industrial parts,'' in \emph{Proceedings of the 1994 IEEE
  International Conference on Robotics and Automation}.\hskip 1em plus 0.5em
  minus 0.4em\relax IEEE, 1994, pp. 1650--1655.

\bibitem{causey1997design}
G.~C. Causey, R.~D. Quinn, N.~A. Barendt, D.~M. Sargent, and W.~S. Newman,
  ``Design of a flexible parts feeding system,'' in \emph{Proceedings of
  International Conference on Robotics and Automation}, vol.~2.\hskip 1em plus
  0.5em minus 0.4em\relax IEEE, 1997, pp. 1235--1240.

\bibitem{berkowitz1996designing}
D.~R. Berkowitz and J.~Canny, ``Designing parts feeders using dynamic
  simulation,'' in \emph{Proceedings of IEEE International Conference on
  Robotics and Automation}, vol.~2.\hskip 1em plus 0.5em minus 0.4em\relax
  IEEE, 1996, pp. 1127--1132.

\bibitem{goldberg1997estimating}
K.~Goldberg, J.~Craig, B.~Carlisle, and R.~Zanutta, ``Estimating throughput for
  a flexible part feeder,'' in \emph{Experimental Robotics IV}.\hskip 1em plus
  0.5em minus 0.4em\relax Springer, 1997, pp. 486--497.

\bibitem{causey1999testing}
G.~C. Causey, R.~D. Quinn, and M.~S. Branicky, ``Testing and analysis of a
  flexible feeding system,'' in \emph{Proceedings 1999 IEEE International
  Conference on Robotics and Automation (Cat. No. 99CH36288C)}, vol.~4.\hskip
  1em plus 0.5em minus 0.4em\relax IEEE, 1999, pp. 2564--2571.

\bibitem{branicky2000modeling}
M.~S. Branicky, G.~C. Causey, and R.~D. Quinn, ``Modeling and throughput
  prediction for flexible parts feeders,'' in \emph{Proceedings 2000 ICRA.
  Millennium Conference. IEEE International Conference on Robotics and
  Automation. Symposia Proceedings (Cat. No. 00CH37065)}, vol.~1.\hskip 1em
  plus 0.5em minus 0.4em\relax IEEE, 2000, pp. 154--161.

\bibitem{chalasani1996algorithms}
P.~Chalasani, R.~Motwani, and A.~Rao, ``Algorithms for robot grasp and
  delivery,'' in \emph{2nd International Workshop on Algorithmic Foundations of
  Robotics}.\hskip 1em plus 0.5em minus 0.4em\relax Citeseer, 1996.

\bibitem{lenstra1981complexity}
J.~K. Lenstra and A.~R. Kan, ``Complexity of vehicle routing and scheduling
  problems,'' \emph{Networks}, vol.~11, no.~2, pp. 221--227, 1981.

\bibitem{helvig1998moving}
C.~S. Helvig, G.~Robins, and A.~Zelikovsky, ``Moving-target tsp and related
  problems,'' in \emph{European Symposium on Algorithms}.\hskip 1em plus 0.5em
  minus 0.4em\relax Springer, 1998, pp. 453--464.

\bibitem{han2018complexity}
S.~D. Han, N.~M. Stiffler, A.~Krontiris, K.~E. Bekris, and J.~Yu, ``Complexity
  results and fast methods for optimal tabletop rearrangement with overhand
  grasps,'' \emph{The International Journal of Robotics Research}, p.
  0278364918780999, 2018.

\bibitem{berbeglia2007static}
G.~Berbeglia, J.-F. Cordeau, I.~Gribkovskaia, and G.~Laporte, ``Static pickup
  and delivery problems: a classification scheme and survey,'' \emph{Top},
  vol.~15, no.~1, pp. 1--31, 2007.

\bibitem{christofides1969algorithm}
N.~Christofides and S.~Eilon, ``An algorithm for the vehicle-dispatching
  problem,'' \emph{Journal of the Operational Research Society}, vol.~20,
  no.~3, pp. 309--318, 1969.

\bibitem{frederickson1976approximation}
G.~N. Frederickson, M.~S. Hecht, and C.~E. Kim, ``Approximation algorithms for
  some routing problems,'' in \emph{17th annual symposium on foundations of
  computer science (sfcs 1976)}.\hskip 1em plus 0.5em minus 0.4em\relax IEEE,
  1976, pp. 216--227.

\bibitem{treleaven2013asymptotically}
K.~Treleaven, M.~Pavone, and E.~Frazzoli, ``Asymptotically optimal algorithms
  for one-to-one pickup and delivery problems with applications to
  transportation systems,'' \emph{IEEE Transactions on Automatic Control},
  vol.~58, no.~9, pp. 2261--2276, 2013.

\bibitem{li1997line}
T.-Y. Li and J.-C. Latombe, ``On-line manipulation planning for two robot arms
  in a dynamic environment,'' \emph{The International Journal of Robotics
  Research}, vol.~16, no.~2, pp. 144--167, 1997.

\bibitem{pardo1995system}
G.~Pardo-Castellote, S.~A. Schneider, and R.~Cannon, ``System design and
  interfaces for intelligent manufacturing workcell,'' in \emph{Proceedings of
  1995 IEEE International Conference on Robotics and Automation}, vol.~1.\hskip
  1em plus 0.5em minus 0.4em\relax IEEE, 1995, pp. 1105--1112.

\bibitem{mattone2000sorting}
R.~Mattone, M.~Divona, and A.~Wolf, ``Sorting of items on a moving conveyor
  belt. part 2: performance evaluation and optimization of pick-and-place
  operations,'' \emph{Robotics and Computer-Integrated Manufacturing}, vol.~16,
  no. 2-3, pp. 81--90, 2000.

\bibitem{schrage1968letter}
L.~Schrage, ``Letter to the editor—a proof of the optimality of the shortest
  remaining processing time discipline,'' \emph{Operations Research}, vol.~16,
  no.~3, pp. 687--690, 1968.

\bibitem{bozma2012multirobot}
H.~I. Bozma and M.~Kalal{\i}o{\u{g}}lu, ``Multirobot coordination in
  pick-and-place tasks on a moving conveyor,'' \emph{Robotics and
  Computer-Integrated Manufacturing}, vol.~28, no.~4, pp. 530--538, 2012.

\bibitem{daoud2014efficient}
S.~Daoud, H.~Chehade, F.~Yalaoui, and L.~Amodeo, ``Efficient metaheuristics for
  pick and place robotic systems optimization,'' \emph{Journal of Intelligent
  Manufacturing}, vol.~25, no.~1, pp. 27--41, 2014.

\bibitem{huang2015robust}
Y.~Huang, R.~Chiba, T.~Arai, T.~Ueyama, and J.~Ota, ``Robust multi-robot
  coordination in pick-and-place tasks based on part-dispatching rules,''
  \emph{Robotics and Autonomous Systems}, vol.~64, pp. 70--83, 2015.

\bibitem{carp2003dynamic}
D.~Carp-Ciocardia \emph{et~al.}, ``Dynamic analysis of clavel's delta parallel
  robot,'' in \emph{2003 IEEE International Conference on Robotics and
  Automation (Cat. No. 03CH37422)}, vol.~3.\hskip 1em plus 0.5em minus
  0.4em\relax IEEE, 2003, pp. 4116--4121.

\bibitem{clavel1988fast}
R.~Clavel, ``A fast robot with parallel geometry,'' in \emph{Proc. Int.
  Symposium on Industrial Robots}, 1988, pp. 91--100.

\bibitem{zsombor2004descriptive}
P.~Zsombor-Murray, ``Descriptive geometric kinematic analysis of {C}lavel\'s
  ``delta'' robot,'' \emph{Centre of Intelligent Machines, McGill University,
  USA}, 2004.

\bibitem{TanYu19ISRR}
W.~N. Tang and J.~Yu, ``Taming combinatorial challenges in optimal clutter
  removal tasks,'' in \emph{Proceedings International Symposium on Robotics
  Research}, 2019.

\bibitem{held1962dynamic}
M.~Held and R.~M. Karp, ``A dynamic programming approach to sequencing
  problems,'' \emph{Journal of the Society for Industrial and Applied
  mathematics}, vol.~10, no.~1, pp. 196--210, 1962.

\bibitem{gonzalez1976open}
T.~Gonzalez and S.~Sahni, ``Open shop scheduling to minimize finish time,''
  \emph{Journal of the ACM (JACM)}, vol.~23, no.~4, pp. 665--679, 1976.

\end{thebibliography}
}

\end{document}